\documentclass[12pt,a4paper]{article}
\usepackage[centertags]{amsmath}
\usepackage{amsfonts,amsthm,amssymb}
\usepackage{amssymb}
\usepackage{amsmath}
\usepackage{graphicx}
\usepackage{graphics}
\usepackage{booktabs}
\usepackage{appendix}
\usepackage[hmargin=3.175cm,vmargin=3.175cm]{geometry}
\usepackage{etoolbox}
\usepackage[authoryear]{natbib}
\usepackage{authblk}
\usepackage{color}

\newcommand{\E}{\mathbb{E}}
\newcommand{\R}{\mathbb{R}}
\newcommand{\Pro}{\mathbb{P}}
\newcommand{\Bay}{r^*}
\newcommand{\LBay}{\tilde{r}^*}
\newcommand{\Lf}{\tilde{F}}
\newcommand{\oL}{\overline{L}}
\newcommand{\one}{\textbf{1}}
\newcommand{\LPri}{\tilde{\mu}}
\newcommand{\Lx}{\tilde{x}}
\newcommand{\Ly}{\tilde{y}}
\newcommand{\Lz}{\tilde{z}}
\newcommand{\Lr}{\tilde{r}}
\newcommand{\0}{\textbf{0}}
\newcommand{\poly}{\textsf{poly}}

\newcommand{\logit}{\textsl{logit}}
\newcommand{\muh}{\hat{\mu}}

\newtheorem{lemma}{Lemma}
\newtheorem{proposition}{Proposition}
\newtheorem{theorem}{Theorem}

\theoremstyle{definition}
\newtheorem{definition}{Definition}
\newtheorem{example}{Example}
\newtheorem{remark}{Remark}[section]

\begin{document}

\title{Learning of Optimal Forecast Aggregation in Partial Evidence Environments}
\author{Yakov Babichenko}
\author{Dan Garber}

\affil{\small Faculty of Industrial Engineering, Technion\textemdash Israel Institute of Technology.}

\maketitle

\begin{abstract}
We consider the forecast aggregation problem in repeated settings, where the forecasts are done on a binary event. At each period multiple experts provide forecasts about an event. The goal of the aggregator is to aggregate those forecasts into a subjective accurate forecast. We assume that experts are \emph{Bayesian}; namely they share a common prior, each expert is exposed to some evidence, and each expert applies Bayes rule to deduce his forecast. The aggregator is \emph{ignorant} with respect to the information structure (i.e., distribution over evidence) according to which experts make their prediction. The aggregator observes the experts' forecasts only. At the end of each period the actual state is realized.
We focus on the question whether the aggregator can learn to aggregate \emph{optimally} the forecasts of the experts, where the optimal aggregation is the Bayesian aggregation that takes into account all the information (evidence) in the system.

We consider the class of \emph{partial evidence information structures}, where each expert is exposed to a different subset of conditionally independent signals. Our main results are positive; We show that optimal aggregation can be learned in polynomial time in a quite wide range of instances of the partial evidence environments. We provide a tight characterization of the instances where learning is possible and impossible.
\end{abstract}

\section{Introduction}

How should an aggregtor aggregate forecasts made by different experts to a single subjective forecast? This question has been studied up to some extent in the literature, see discussion in Section \ref{sec:rl}. In order to address this forecast aggregation problem \emph{theoretically}, we first should specify:

\begin{itemize}
\item[(a)] How do experts make their forecasts?
\item[(b)] What does the aggregator know about the problem?
\item[(c)] How do we measure the accuracy of a forecast? 
\end{itemize}

\noindent
\textbf{How do experts make their forecasts?}
We adopt the standard Bayesian modelling of the problem. More concretely,
we assume that experts are \emph{Bayesian} and \emph{honest}. 

There is some future binary event $\Omega=\{0,1\}$. Experts share a common prior $\mu \in \Delta(\Omega)$ about the future event. Experts are exposed to different (possibly correlated) evidence which has correlation with the future event. Henceforth, we use the standard terminology of \emph{signal} to refer to an evidence. Each expert applies Bayes role to deduce her subjective posterior belief about the event, based on the signals that she observes (i.e., experts are Bayesian). Each expert reports truthfully her subjective posterior probability (i.e., experts are honest).

\noindent
\textbf{What does the aggregator know about the problem?}
The aggregator is ignorant with respect to the structure according to which the signals are drawn, and he does not observe the signals. Aggregator observes experts' forecasts only. I.e., aggregator does not know the evidence that led each one of the experts to his subjective forecast, only the forecasts themself. We distinguish between \emph{prior-ignorant aggregator} (one that does not know the prior) and \emph{prior-aware aggregator} (one that shares the common prior with the experts).

The distinction between the \emph{ignorant aggregator} and the \emph{Bayesian experts} is realistic in many scenarios. For instance, an ordinary web user who searches for a weather forecast has no understanding in weather. The only available information for him is the forecasts of the forecasting agencies (experts). Forecasting agencies, on the other hand, base their forecasts on evidence and have a precise models of weather forecasting. Thus, it is plausible to assume that their forecast is done in a much more rational manner; I.e., in a Bayesian manner.  

The problem of \emph{one-shot aggregation} without any knowledge of the aggregator is known to be impossible (in a very strong sense) in case where the number of experts is large, \cite{ABS}. In this paper we consider a \emph{repeated} setting. Initially, indeed the aggregator has no information about the problem. At the end of each period aggregator gets to see the realized state, and thus can extract some information about the problem, and about the abilities of each expert to predict. 

\noindent
\textbf{How do we measure the accuracy of a forecast? And what is an \emph{optimal} forecast?}
A standard measurement of forecasts accuracy is done by \emph{proper scoring rules}, i.e., a function from a forecast and a realized state into scores (payments) that incentivizes Bayesian agents to report their belief. The two most common scoring rules are the \emph{square loss} and the \emph{logarithmic loss}. In this paper we focus on the later. An \emph{optimal} forecast (i.e., the one with the minimal loss) is the forecast that aggregates \emph{all} the signals in the system in a Bayesian manner. Note that in order to make an optimal forecast, the aggregator should have access to all signals and he should know the distribution according to which these signals are drawn. This information is not available to the aggregator in our setting.

In this paper we focus on the question whether the aggregator can \emph{learn} to aggregate forecasts optimally, after sufficiently many periods? If so, how? And how fast can he learn to do so? 

We emphasize that unlike standard \emph{expert advice problem} \cite{CBL} where typically the benchmark is the best expert, here we set \emph{the most challenging} benchmark; I.e., we compare the performance of the aggregator with \emph{the optimal forecast at each period}. Typically the best aggregated forecast will defer from the forecasts of all experts, and in particular the best expert performance will be inferior to the performance of the optimal-forecast at each period.

\subsection{Partial evidence environment}
It is hard to model what precisely is the information that is available to experts in realistic forecasting problems. Most of the Bayesian theoretical literature significantly simplifies the environment, by assuming that experts are exposed to \emph{independent} information conditional on the state, e.g., \cite{BHW,Jack}. This abstraction of the problem indeed make the information aggregation problem much more tractable. However, in some environments this assumption might be too demanding. For instance, two experts of weather forecast may use the same measurement apparatus in their forecast model. Experts in the political arena might evaluate candidates chances of success in the elections according to their publicly observed attributes (diversity in the evaluation might be caused by different attributes that experts take into account in their model). On the other extreme, allowing arbitrary correlation among experts evidence make the problem of forecast aggregation not tractable\footnote{Even in the simple case of binary signals for each expert, just the description of an arbitrary correlated distribution over the profiles of signals is exponential in the number of experts. Moreover, even for the simplest case of two experts with binary signals, the impossibility result of \cite{ABS} Proposition 1 for the \emph{one shot} aggregation problem can be easily extended to the same impossibility result for the \emph{repeated} forecast aggregation problem.}. Thus we suggest an intermediate model of \emph{partial evidence}, which on the one hand is not too simplified: it allows correlation between experts which might be caused due to exposure to the same evidence. On the other hand, the model is not too complicated, which will make the problem tractable (at least in some cases). Very similar model of \emph{information diversity} has been studied recently by \cite{SPU,EPSU}, see a discussion in Section \ref{sec:rl}. 

The partial evidence model consists of a finite set of conditionally independent \emph{signals} $[m]$ (so far we are aligned with the standard literature). Unlike the standard literature, we assume that each expert $i\in [n]$ is exposed to a \emph{subset} of these signals, and uses only these signals to calculate his posterior belief. We denote by $A\in \{0,1\}^{n\times m}$ the evidence matrix, where $A(i,j)=1$ indicates that expert $i$ observes signal $j$. We recall that our aggregator does not know how many signals there are, what is the distribution over the signals, and does not know $A$.

\subsection{Summary of Results}
We consider two settings, of a \emph{static} and a \emph{dynamic} environment. In both environments the evidence matrix $A$ (and also the set of experts and set of signals) remains fixed. This assumption correspond to the fact that forecasting agencies do not frequently change the \emph{model} according to which they produce their forecast. 

In the static environment the information structure remains fixed along the process. This scenario may correspond for instance to scenarios of medical prognosis where initially, before any information has been exposed, the problem is identical. We show that if the evidence matrix $A$ is \emph{not injective}, even the prior-aware aggregator cannot learn the optimal aggregation (see Proposition \ref{pro:imp}) irrespective of the number of samples. The main positive result shows that whenever $A$ is \emph{injective}, even the prior-ignorant aggregator can learn the optimal aggregation; Moreover, he can do so after a polynomial (in $n$ and $m$) number of samples (see Theorem \ref{th:stat}).

In the dynamic environment, the information structure changes along time (in an adversarial manner). 
For instance, in the weather forecasting example, the probability distribution of rain/no rain and the distribution of outcomes of the measurement apparatuses varies during the seasons of the year.
We restrict attention to the case where $A$ is injective (because otherwise learning is impossible even in the static setting). We show that the \emph{prior-ignorant} aggregtor cannot learn the optimal aggregation (see Proposition \ref{pro:imp-dyn}) irrespective of the number of samples. The main positive result shows that the \emph{prior-aware} aggregator can learn the optimal aggregation; Moreover, he can do so after a polynomial number of samples (see Theorem \ref{th:dyn}).

\begin{table}[h]
\caption{The possible and the impossible in learning of optimal forecast aggregation with respect to the three relevant parameters of the problem: whether the evidence matrix is injective or not, whether the aggregator is prior-aware or not, and whether the environment is static or dynamic.}
\resizebox{\columnwidth}{!}{
\begin{tabular}{ccccccc}
\multicolumn{3}{c}{Injective $A$}                                                                                &  & \multicolumn{3}{c}{Non injective $A$}                                                                    \\
                                         & prior-aware                       & prior-ignorant                    &  &                                          & prior-aware                   & prior-ignorant                \\ \cline{2-3} \cline{6-7} 
\multicolumn{1}{c|}{static}  & \multicolumn{1}{c|}{$\checkmark$} & \multicolumn{1}{c|}{$\checkmark$} &  & \multicolumn{1}{c|}{static}  & \multicolumn{1}{c|}{$\times$} & \multicolumn{1}{c|}{$\times$} \\ \cline{2-3} \cline{6-7} 
\multicolumn{1}{c|}{dynamic} & \multicolumn{1}{c|}{$\checkmark$} & \multicolumn{1}{c|}{$\times$}     &  & \multicolumn{1}{c|}{dynamic} & \multicolumn{1}{c|}{$\times$} & \multicolumn{1}{c|}{$\times$} \\ \cline{2-3} \cline{6-7} 
\end{tabular}
}
\end{table}

\subsection{Techniques}
The positive results are based on the following idea. The model, as given, is non-linear. However, once we translate the problem into the log-likelihood space (i.e., we translate experts' forecasts to log-likelihoods), then the desired Bayesian aggregation (of all the information) becomes a liner problem. The (minus) logarithmic loss in the log-likelihood representation turns out to be a convex function. Thus, we can leverage tools from the paradigm of \textit{online convex optimization} \cite{Hazan16}, mainly a regret-minimization algorithm known as \textit{online gradient descent} (OGD), which we apply in order to converge to the optimal \emph{linear} aggregation. The following two obstacles arise in the application of OGD. The first is that OGD requires as feedback the gradient of the loss function on each round, which is not available in our setting. We overcome this difficulty by replacing the loss function of interest with an unbiased estimator, for which the gradient vector is known, and thus, we are able to obtain sublinear regret bounds which hold in expectation. The second obstacle, is that the input to the algorithm at each period is \emph{not necessarily bounded}; Although forecasts are bounded in $[0,1]$, their \emph{log-likelihood} is not necessarily bounded, which may potentially cause the gradient feedback to OGD to explode. We overcome this obstacle by setting an appropriate threshold $\tau$ to the forecasts. We treat differently the cases where all experts provide forecasts in the segment $[\tau,1-\tau]$, and the case where at least one expert provide an ``extreme" forecast. In the former case we simply apply the stochastic gradient decent. In the latter we provide a forecast that is similar to that of the ``extreme" expert.

The translation of the problem to the log-likelihood space inherently assumes the knowledge of the prior. Although in the static model we consider a prior-ignorant aggregator, this aggregator may provide a good \emph{estimation} for the prior simply by counting the fraction of rounds where the realization turned out to be $\omega=1$. We show that good estimation for the prior is sufficient in order to apply our techniques, which implies the result in the static model for the prior-ignorant aggregator. 

\subsection{Related Literature}\label{sec:rl}
The research question in this paper is in the interplay of \emph{machine learning} and \emph{forecast aggregation}. Below we discuss the relevant literature in this two areas.

\subsubsection{Forecast aggregation}
Many different heuristics to aggregate forecasts have been suggested in the literature; linear heuristics \cite{Stone,GM,DM}, multiplicative heuristics \cite{Bordley}, and other non-linear heuristics \cite{Armstrong,PJ,PFJ,RG,SFU}. These heuristics are either optimal on some concrete information structure, or have the convenience of simplicity. 
The empirical approach to forecast aggregation typically starts with a class of simple (possibly parametric) heuristics, and checks which heuristic preforms optimally on the training set (i.e., optimizes the parameters). Substantial part of the forecast aggregation literature adopts the empirical approach, see \cite{BD} for an excellent survey of the empirical literature, and the more recent studies \cite{UMS,FM}.
However, it is not clear which class of heuristics we should start with at the first place, and there are no theoretical guarantees for the performance of these heuristics. Our positive results suggest an interesting class of initial (relatively simple) heuristics: linear heuristics over the log-likelihood space with a ``prior adjustment", see details in Sections \ref{sec:ll-trans} and \ref{sec:sgd}. This class of heuristics has a theoretical justification: for a quite wide class of information structures - the partial evidence class - the optimal aggregation belongs to this simple class of heuristics. 

The \emph{partial evidence} environment has been suggested previously in \cite{SPU,EPSU} to capture somewhat ``realistic" environments of information. \cite{SPU,EPSU} focus on the \emph{Bayesian aggregation} problem in such an environment where the signals are real numbers, and the binary event is whether the sum of all numbers is positive/negative. Unlike our paper, \cite{SPU,EPSU} focus on an aggregator that \emph{knows} precisely the information structure.    

The setting of an \emph{ignorant aggregator} that aggregates forecasts of \emph{Bayesian experts} has been introduced in \cite{ABS,WiseC}. Unlike our paper, \cite{ABS,WiseC} focus on the \emph{one-shot} aggregation problem. \cite{ABS} indicate that the one-shot aggregation problem for large number of experts is impossible. Our positive results in this paper indicate that in a repeated setting \emph{the optimal} aggregation is possible in quite wide class of information structures.

\subsubsection{Machine Learning}
The problem of sequential prediction with adversarial data has been studied extensively in the machine learning literature in the past couple of decades, mainly under the title \textit{online learning}, see for instance the excellent text \cite{CBL}. The best-known problem in this setting is known as \textit{prediction with expert advice} \cite{littlestone1994weighted,Vovk90}, in which the typical benchmark is the best expert in hindsight. A variety of advice-aggregation algorithms that attain this benchmark has been suggested. In particular regret minimizing algorithms. The most competitive benchmark is that of \emph{best decision in each period}, which obviously is unattainable in general adversarial settings. The standard literature in the expert advice problem assumes that expert advices are \emph{adversarial}. In this paper we impose some unknown to the aggregator \emph{Bayesian structure} on the advices of the experts, and ask weather in such environments decision maker can achieve the best-decision in each period benchmark. 


\section{Model}

The unknown state of nature is $\Omega=\{0,1\}$ with the common prior $\mu=\Pro(\omega=1)$.
There are $m$ signals $s_1,...,s_m$ in the system with possible values $s_i\in S_i$, these are the evidence in the system. E.g., in the weather forecasting case there are measurement apparatus on which the forecasting models are based, in the medical prognosis case there are the medical tests of the patient.
The correlation between $\omega$ and the signal $s_j\in S_j$ is given by the distribution $C_j\in \Delta(\Omega \times S_j)$, when $C_j(\omega=1)=\mu$. Every signal $s_j\in S_j$ defines a posterior belief about $\omega$, $x_j(s_j)=\Pro(\omega=1|s_j)\in [0,1]$. The distribution $C_j$ defines two conditional distributions $C_j^0$ and $C_j^1$. 
We assume that conditioned on $\omega$, $C^\omega_j$ is independent of $C^\omega_{j'}$ for $j\neq j'$. We denote by $s=(s_1,...,s_m)$ the profile of signals, and by $C$ the joint distribution over $(\omega,s)$.

It is well known (see e.g., \cite{Bordley}) that the Bayesian formula for aggregation of conditionally independent signals is given by:
\begin{align}\label{eq:bay-arg}
\Bay(s)=\Pro(\omega=1|s_1,...,s_m)=\frac{(1-\mu)^{m-1}\Pi_{j=1}^{m} x_j(s_j)}{(1-\mu)^{m-1}\Pi_{j=1}^{m} x_j(s_j)+\mu^{m-1} \Pi_{j=1}^{m} (1-x_j(s_j))}.
\end{align}
Note that $\Pro(\omega=1|s_1,...,s_m)$ is, equivalently, the optimal forecast in the given environment, that aggregates \emph{all} the available information in the system.

Expert $i\in [n]$ observes only the signals $(s_j)_{j\in A_i}$ where $A_i\subset[m]$.
We denote by $A\in \{0,1\}^{n\times m}$ the \emph{evidence matrix}, where $A(i,j)=1$ iff $j\in A_i$.
Since expert $i$ is Bayesian and honest, her posterior belief (and also forecast) is given by:
\begin{align}\label{eq:bay-fore}
F_i(s)=\Pro(\omega=1|(s_j)_{j\in A_i})=\frac{(1-\mu)^{|A_i|-1}\Pi_{j\in A_i} x_j(s_j)}{(1-\mu)^{|A_i|-1}\Pi_{j\in A_j} x_j(s_j)+\mu^{|A_i|-1} \Pi_{j\in A_i} (1-x_j(s_j))}.
\end{align}
We denote by $F=(F_i)_{i\in [n]}\in [0,1]^n$ the profile of forecasts of all experts. This is precisely the object that is observed by the aggregator, according to which he should determine his own subjective forecast $r\in [0,1]$. We measure the aggregator's performance with respect to the \emph{logarithmic loss}:
\begin{align*}
l(r,\omega)=\begin{cases}
-\ln(r) &\text{ if } \omega=1 \\
-\ln(1-r) &\text{ if } \omega=0.
\end{cases}
\end{align*}
The aggregator's objective is to minimize his \emph{expected} logarithmic loss, which is given by
$L(y)=-\Bay(s)\ln(r)-(1-\Bay(s))\ln(1-r)$.
The expected loss $L(y)$ is minimized at $r=\Bay(s)$, because logarithmic loss is a proper scoring rule. Unfortunately, the aggregator does not know $\Bay(s)$, because the only information available to him is $F=F(s)$. We set the most competitive benchmark for the aggregator: he should minimize the regret $R(y)=L(r)-L(\Bay(s))$.

We consider a repeated setting, where on each period $t$, the aggregator observes the current profile of forecasts $F_t$ and the history of the past forecasts and realizations $(F_{t'},\omega_{t'})_{t'<t}$. The aggregator provides a forecast $r_t$ (before the actual state is realized). We distinguish between two environments. In the \emph{static environment} the partial evidence information structure $(A,C)$ remains fixed. In the \emph{dynamic environment} the information structure $C$ varies over time, but $A$ remains fixed. Namely the information structure along time is a sequence $(A,C_t)_t$. 

\begin{definition}
We say that \emph{the aggregator can grantee a total regret of $R(T)$ within $T$ periods in the static (dynamic) environment}\footnote{Alternatively, one may define learning with respect to \emph{sample complexity} rather than repeated settings; I.e., how many samples are needed in order to make a good forecast at a given period. Our techniques (for the positive and the negative results) can be easily modified to address this latter definition using, by-now a standard, \textit{online-to-batch} conversion technique, see for instance \cite{Hazan16} (Chapter 9).} if there exists an algorithm for the aggregator, that guarantees $$\sum_{t\in T} \E_{s_t \sim C_t, \omega_t \sim C_t|s_t}(l(r_t,\omega_t)-l(\Bay(s_t),\omega_t))\leq R(T)$$ for all partial evidence information structures $(A,C)$ (for all sequences of partial evidence information structures $(A,C_t)_{t\in [T]}$). We say that \emph{aggregator can learn the optimal aggregation}, if the aggregator can grantee a total regret of $R(T)=o(T)$ within $T$ periods.
\end{definition}

\section{Results}
\subsection{Static Environment}\label{sec:stat}
%
%

We start with an example that demonstrates that learning of optimal aggregation is not always possible (even for very large/ infinite number of samples).

\begin{example}\label{ex:imp}
Consider a scenario where the prior is $\mu=\frac{1}{2}$. We have three conditional i.i.d. signals with outcome $s_i=0,1$ where $s_i$ is the correct state with probability $\frac{3}{4}$. Note that $x_i(0)=\frac{1}{4}$ and $x_i(1)=\frac{3}{4}$. We have two experts. The first observes the signals $\{s_1,s_2\}$, and the second $\{s_2,s_3\}$. 

For the sake of argument, assume that the aggretor \emph{knows} all of the above parameters of the problem. It is not obvious that the aggregator indeed can learn all of the above parameters precisely (or approximately) after sufficiently many periods; but this is an \emph{upper bound} on the information that the aggregator can extract from a sequence of realizations.

With probability $\frac{1}{2}\cdot\frac{3}{4}\cdot\frac{1}{4}\cdot\frac{3}{4}+
\frac{1}{2}\cdot\frac{1}{4}\cdot\frac{3}{4}\cdot\frac{1}{4}=\frac{3}{32}$ the realization of signals will be $(0,1,0)$. In this case both experts will forecast $\frac{1}{2}$, whereas the optimal aggregation is $\Bay((0,1,0))=\frac{1}{4}$.

With same probability $\frac{3}{32}$, the realization of signals will be $(1,0,1)$. In this case, again, both experts will forecast $\frac{1}{2}$, whereas the optimal aggregation is $\Bay((1,0,1))=\frac{3}{4}$.

Note that the aggregator in both cases observes $F=(\frac{1}{2},\frac{1}{2})$. It is easy to check that the best aggregator can do is to forecast $r=\frac{1}{2}$ when he observes the vector $(\frac{1}{2},\frac{1}{2})$, therefore his expected regret is at least 
$$\frac{6}{32}\left[\frac{1}{4}\ln(\frac{1}{4}))+\frac{3}{4}\ln(\frac{3}{4})-\frac{1}{4}\ln(\frac{1}{2})-\frac{3}{4}\ln(\frac{1}{2})\right]\approx 0.024.$$
\end{example}

The failure of aggregation in this example is a result of the fact that the evidence matrix
\begin{align*}
A=\begin{bmatrix}
    1 & 1 & 0 \\
    0 & 1 & 1  
\end{bmatrix}
\end{align*}
is not injective. This example can be generalized to the following impossibility result about learning of optimal aggregation in the case where the evidence matrix is not injective.

\begin{proposition}\label{pro:imp}
If\footnote{Whether aggregator can learn the optimal aggregation in the esoteric case where $\ker(A)\neq \{\0_m \}$ (i.e., $A$ is not injective), but $\ker(A)\subset \{z\in \R^m: \sum_i z_i =0 \}$ remains an open question.} $\ker(A)\setminus \{z\in \R^m: \sum_i z_i =0 \} \neq \emptyset$ then the prior-aware aggregator cannot learn the optimal aggregation (even) in a static environment.
\end{proposition}


The idea of the proof is similar to the example. We generate two instances where the optimal aggregation defers in these two instances, but the aggregator observes exactly the same profile of forecasts. Such a construction is possible because $A$ maps two different vectors to one (not injective). The formal proof is relegated to Section \ref{sec:pr}. 

Our first main positive result states that whenever $A$ is injective, the aggregator can learn to predict optimally for a large enough number of samples. Moreover, if $A$ is "injective enough" (which is measured by its minimal singular value), then a dataset of size $\poly(n)$ is sufficient for learning of optimal aggregation.

We denote by $\sigma_{\min}(A)$ the minimal singular value of $A$, i.e., $\sigma_{\min}(A)^2$ is the minimal eigenvalue of the matrix $A^{\top}A$. Note in particular, that since A is injective, $\sigma_{\min}(A) > 0$.

\begin{theorem}\label{th:stat}
In case $A$ is injective, the prior-ignorant aggregator can learn the optimal aggregation in a static environment. Moreover, for the class of information structures $(A,C)$ such that $\sigma_{\min}(A)\geq \sigma>0$, the prior-ignorant aggregator can guarantee a total regret of $R(T)=\tilde{O}(n \sigma^{-1} \sqrt{T})$ within $T$ periods.
\end{theorem}

\begin{remark}\label{rem:sv}
The minimal singular value effects the rate of learning. There are matrices with low (or even zero) singular value where the rate of learning might be slow (or impossible as stated in Proposition \ref{pro:imp}). However we may ask what is the singular value in a "typical" matrix. A standard approach to address this question is to consider the class of all $n\times m$ 0/1 matrices and ask how many of them has small minimal singular value (i.e., the learning is slow), which is equivalent to assuming that each node is an i.i.d. Bernoulli random variable with $p=\frac{1}{2}$. 
Singular values of random matrices has been extensively studied. In particular, \cite{TG} (see Section III.C) have proved that in case $\frac{m}{n}\leq C <1$ (namely we have substantially less signals than experts) with probability close to 1 all singular values are greater than $\Omega(\sqrt{n})$, or equivalently $\sigma^{-1}=O(\frac{1}{\sqrt{n}})$. Interestingly, in such matrices the dependence of the singular value \emph{improves} the rate of learning. The total regret in such matrices can be bounded by $R(T)=\tilde{O}(\sqrt{nT})$.

\end{remark}

\subsection{Dynamic Environment}\label{sec:dyn}
%
%
%

In Proposition \ref{pro:imp} we saw that injectiveness of $A$ is a necessary condition for learning of optimal aggregation. This observation obviously holds also for the harder dynamic environment. In Theorem \ref{th:stat} we saw that the \emph{prior-ignorant} aggregator is able to learn an optimal aggregation in the static environment. The following Proposition shows that this is no longer the case for the dynamic environment.

\begin{proposition}\label{pro:imp-dyn}
The prior-ignorant aggregtor cannot learn the optimal aggregation in a dynamic environment, even for the case where $A=\textsf{Id}_n$ (i.e., experts receive conditionally independent signals). Moreover, the aggregtor cannot guarantee a regret below $T\ln 2$ within $T$ periods\footnote{A regret of at most $T\ln 2$ can be trivially guaranteed by always forecasting $r\equiv \frac{1}{2}$, irrespective of experts' forecasts.}.

\end{proposition}


Our second main positive result states that the \emph{prior-aware} aggregator, can learn the optimal aggregation even in the dynamic environment.

\begin{theorem}\label{th:dyn}
In case $A$ is injective, the prior-aware aggregator can learn the optimal aggregation in a dynamic environment. Moreover, for the class of information structures $(A,C)$ such that $\sigma_{\min}(A)\geq \sigma$ the prior-aware aggregator can guarantee a total regret of\footnote{See Remark \ref{rem:sv} for a discussion on the dependence on the minimal singular value.} $R(T)=\tilde{O}(n \sigma^{-1} \sqrt{T})$ within $T$ periods.
\end{theorem}

\section{Proofs}\label{sec:pr}

We start with the proofs of the main positive results (Theorems \ref{th:stat} and \ref{th:dyn}). Thereafter, we present the proofs of the negative results (Propositions \ref{pro:imp} and \ref{pro:imp-dyn}).

The proofs involve quite a few notations. For the reader's convenience we include in Table \ref{tb:notations} the notations that we use along the proofs with short description for the notations. The formal definitions are given along the proofs.

\begin{table}
\caption{Summary of notations along the proof.}\label{tb:notations}
\resizebox{\columnwidth}{!}{
 \begin{tabular}{| c | l |} 
 \hline
 $n$ & Number of experts. \\ 
 \hline
 $m$ & Number of signals. \\
 \hline
 $\mu\in [0,1]$ & The prior. \\
 \hline
 $\tilde{\mu}\in \R$ & The log-likelihood of the prior. \\
 \hline
 $\hat{\mu}\in [0,1]$ & Estimator for $\mu$. \\
 \hline
 $\hat{\tilde{\mu}}\in \R$ & Estimator for $\tilde{\mu}$. \\
 \hline
 $s=(s_i)_{i\in [m]}$ & Signals profile.\\
 \hline
 $r\in [0,1]$ & Aggregator's forecast.\\
 \hline
 $\tilde{r}\in R$ & Aggregator's forecast in the log-likelihood space.\\
 \hline
 $\Bay(s)\in [0,1]$ & The optimal Bayesian aggregation of all signals.\\
 \hline
 $\LBay(s)\in \R$ & Log-likelihood of the optimal Bayesian aggregation.\\
 \hline
 $F_i=F_i(s)\in [0,1]$ & Forecast made by expert $i\in [n]$.\\
 \hline
 $\tilde{F}_i=\tilde{ F}_i(s)\in \R$ & Log-likelihood of the forecast made by expert $i\in [n]$.\\
 \hline
 $\tilde{F}\in \R^n$ & Profile of log-likelihoods of the forecasts.\\
 \hline
 $x_i=x_i(s_i)\in [0,1]$ & Posterior belief, after observing a single signal $s_i$.\\
 \hline
 $\tilde{x}_i=\tilde{x}_i(s_i)\in \R$ & Log-likelihood of the Posterior belief, after observing a single signal $s_i$.\\
 \hline
 $\tilde{y}_i=\tilde{x}_i - \tilde{\mu}$ & Prior adjusted log-likelihood induced by the signal $s_i$.\\
 \hline
 $\tilde{y}=(\tilde{y}_i)_{i\in [m]}\in \R^m$ & Profile of prior adjusted log-likelihoods induced by all signals.\\
 \hline
 $\hat{y}\in \R^m$ & Estimator for $\tilde{y}$.\\
 \hline
 $\tilde{z}=A\tilde{y}\in \R^n$ & A $\tilde{\mu}$-shift of the observed forecast profile in the log-likelihood space.\\
 $\tilde{z} = \tilde{F}-\tilde{\mu}\one_n $ & \\
 \hline
 $\hat{z}\in \R^n$ & Estimator for $\tilde{z}$.\\
 \hline
 $h\in \R^n$ & Elements of the hypothesis class.\\
 \hline
 $h^*\in \R^n$ & Optimal hypothesis which corresponds to the optimal aggregation.\\
 \hline
 $L(r,s)=L(r)\in [0,\infty)$ & Aggregator's expected loss conditional on $s$.\\
 \hline
 $\tilde{L}(\tilde{r},s)=\tilde{L}(\tilde{r}) \in [0,\infty)$ & The analogue of the loss $L$ in the log-likelihood space.\\
 \hline
 $\overline{L}(\tilde{r},s)=\overline{L}(\tilde{r})\in [0,\infty)$ & The analogue of the loss $L$ in the hypothesis class.\\
 \hline
\end{tabular}
}
\end{table}

\subsection{Translation of the problem to the log-likelihood space}\label{sec:ll-trans}
The model, as presented, is non-linear. However, it is equivalent to a linear model by applying several (quite standard) operations.

First note that if we focus on the \emph{likelihoods} of the events (rather than their probabilities) equations \eqref{eq:bay-arg} and \eqref{eq:bay-fore} are equivalent to
\begin{align}\label{eq:like}
\frac{\Bay(s)}{1-\Bay(s)}=\frac{(1-\mu)^{m-1}\Pi_{j=1}^{m}x_j(s_j)}{\mu^{m-1}\Pi_{j=1}^{m}(1-x_j(s_j))} \text{ and } \frac{F_i(s)}{1-F_i(s)}=\frac{(1-\mu)^{|A_i|-1}\Pi_{j\in A_i}x_j(s_j)}{\mu^{|A_i|-1}\Pi_{j\in A_i}(1-x_j(s_j))}. 
\end{align}
We denote $\LPri=\ln(\frac{\mu}{1-\mu})$ and $\Lx_j=\ln(\frac{x_j(s_j)}{1-x_j(s_j)})$. By applying logarithm on the expressions of equation \eqref{eq:like} we get that
\begin{align*}
\LBay(s)&:=\ln(\frac{\Bay(s)}{1-\Bay(s)})=\sum_{j=1}^m \Lx_i - (m-1) \LPri = \sum_{j=1}^m (\Lx_i - \LPri) + \LPri  \text{ and }\\
\Lf_i(s)&:=\ln(\frac{F_i(s)}{1-F_i(s)})=\sum_{j\in A_i} (\Lx_i - \LPri) + \LPri.
\end{align*}
We denote $\Lf(s)=(\Lf_i(s))_{i\in [n]}\in \R^n$, $\Ly=(\Lx_j - \LPri)_{j\in [m]}\in \R^m$ and $\one_k=(1,...,1)\in \R^k$ (for any $k$) and we get the following linear expressions:
\begin{align*}
\LBay(s)=\Ly\cdot \one_m +\LPri \text{ and } \Lf(s)=A\Ly + \LPri \one_n.
\end{align*}
Now, the idea is to formulate the original forecast aggregation problem in terms of the, hidden from the aggregator, vector $\Ly$, hidden from the aggregator, matrix $A$, and the observed by the aggregator vector of forecasts $\Lf= A\Ly + \LPri \one_n$.

In the original problem, the aggregator observes the vector $F(s)$, and generates a forecast $r\in [0,1]$. In the equivalent formulation of the problem the aggregator observes $A\Ly \in \R^n$, and he reports a log-likelihood number $\Lr\in \R$, where the optimal report is $\LBay =\Ly\cdot \one_m + \LPri \in \R$. Note that the inverse function of the log-likelihood $\tilde{w}=\ln(\frac{w}{1-w})$ is the logit function $w=\logit(\tilde{w})=(1+\exp(-\tilde{w}))^{-1}$. Therefore, the aggregator's logarithmic loss is given by

\begin{align*}
\tilde{L}(\Lr)&=-\Bay(s)\ln(\frac{1}{1+\exp(-\Lr)})-(1-\Bay(s))\ln(\frac{\exp(-\Lr)}{1+\exp(-\Lr)}) \\
&= -\Bay(s)\Lr-\ln(\frac{\exp(-\Lr)}{1+\exp(-\Lr)}) \\
&=-(\Bay(s)-1)\Lr + \ln(1+\exp(-\Lr))
\end{align*}

Summarizing, the original forecast aggregation problem is equivalent to the following problem.

\begin{definition}\label{def:ll}
The problem \emph{log-likelihood aggregation} at time $t$ is given by

\noindent
\textbf{Input:} A history of past realizations  $(A\Ly_{t'}+ \LPri_{t'} \one_n,\omega_{t'})_{t'<t}$, where $\omega_{t'}$ is a Bernuli random variable with probability of success $\Bay_{t'}=(1+\exp(-\Ly_{t'} \cdot \one_m -\LPri_{t'}))^{-1}$, and the current vector $A\Ly_t + \LPri_t \one_n \in \R^n$.

\noindent
\textbf{Loss:} The aggregator chooses an action $\Lr_t\in \R$ and suffers an expected loss of $\tilde{L}(\Lr_t)$.
\end{definition}

\subsection{Proof of Theorem \ref{th:dyn}}
Since we assume the aggregator is prior-aware, on each round $t$ he knows the sequence $(\mu_i)_{i\leq{}t}$. For simplicity, we assume that the at each round $t$ the prior $\mu_t=\Pro(\omega_t=1)$ is bounded in $[\beta,1-\beta]$ for a constant\footnote{Otherwise, in cases where the prior is very extreme, one can apply similar (and in fact simpler) arguments to those of Section \ref{sec:ext} to deal with such cases. For clarity of the proof we omit these arguments here.} $\beta>0$.

We set $\tau=T^{-1/2}$. Our algorithm treats differently profiles $F_t\in [0,1]^n$ that contain at least one extreme forecast $F_{i,t}\in [0,\tau) \cup (1-\tau,1]$ (which we call \emph{extreme realizations}), and profiles $F_t\in [\tau,1-\tau]^n$ that do not (which we call \emph{non extreme realizations}). If we prove the following two Claims:
\begin{enumerate}
\item Within $T$ periods where \emph{all} realizations are extreme, the total regret is at most $\tilde{O}(n\sqrt{T})$, and 
\item Within $T$ periods where \emph{all} realizations are non extreme, the total regret is at most $\tilde{O}(n\sigma^{-1}\sqrt{T})$,
\end{enumerate}
then it suffices for the proof of the Theorem, because obviously the sum of these to regrets is an upper bound on the regret in our problem. We start with the proof of Claim (2). 

\subsubsection{Algorithm for non extreme realizations}\label{sec:sgd}

In case $A$ is injective and the aggregator is prior-aware the problem of log-likelihood aggregation (see Definition \ref{def:ll}), which is equivalent to the original forecast aggregation problem, can be further reformulated as follows.

Since the aggregator knows the prior he can deduce the vector $\Lz := A\Ly$ from the observed vector $A\Ly+\LPri \one_n$.

In case $A$ is injective it holds that $A^{\top}A$ is invertible, and hence, denoting the left inverse of $A$ by $A_l^{-1}:=(A^{\top}A)^{-1}A^{\top}$, the optimal aggregation has the form
\begin{eqnarray*}
\one_m\cdot\Ly + \LPri = \one_m^{\top}A_l^{-1}A\Ly+\LPri = (A_l^{-1\top}\one_m)\cdot\Lz+\LPri = h^*\cdot\Lz+\LPri,
\end{eqnarray*}
where we denote $h^*: =\one_m A_l^{-1}\in \R^n$ - the optimal aggregation of $\Lz$, and now the equivalent problem can be written as follows.

\noindent
\textbf{Input:} A history of past realizations $(\Lz_{t'},\omega_{t'})_{t'<t}$, where $\omega_{t'}$ is a Bernoulli random variable with probability of success $\Bay_{t'} = (1+ \exp(-h^*\Lz_{t'}-\LPri_{t'}))^{-1}$, and the current vector $\Lz_{t}\in \R^n$.

\noindent
\textbf{Loss:} The aggregator chooses a vector $h_t \in \R^n$ and suffers an expected loss of $\tilde{L}(h_t\cdot\Lz_t+\tilde{\mu_t})$.

Consider the \emph{linear} hypothesis class $H=\{\tilde{z}\rightarrow{}h\cdot\tilde{z}: h\in \R^n,~\Vert{h}\Vert\leq W\}$. Note that for $W \geq \Vert{A_l^{-1\top}\one_m}\Vert$, the optimal aggregation vector $h^*$ belongs to this class. 
 
Consider the following loss function over a specific instance $(\tilde{z}_t,\omega_t)$ given by   

\begin{align}\label{eq:loss}
\oL_t(h) :=  (1-\omega_t)(h\cdot\tilde{z}_t+\tilde{\mu}_t) + \ln(1 + \exp(-h\cdot\tilde{z}_t - \tilde{\mu}_t)).
\end{align}
Note that since for all $t\in[T]$, $\E[\omega_t\,|\,\tilde{z}_t] = \Bay(\tilde{z}_t) := (1+ \exp(-h^*\Lz_t-\LPri_t))^{-1}$, it follows that
\begin{align}\label{eq:expectedLoss}
\forall h\in\R^n,t\in[T]: \quad \E_{\omega_t}[\oL_t(h)] = \tilde{L}(h\cdot\tilde{z}_t+\tilde{\mu}_t)].
\end{align}

Note that the gradient vector of $\oL_t(h)$ is given by 
\begin{eqnarray*}
\nabla\oL_t(h) = (1- \omega_t)\tilde{z}_t - \frac{\exp(-h\cdot\tilde{z}_t - \tilde{\mu}_t)}{1 + \exp(-h\cdot\tilde{z}_t - \tilde{\mu}_t)}\tilde{z}_t.
\end{eqnarray*}

Note also that for all $t\in[T]$, $h\in\R^n$, since $\omega_t\in\{0,1\}$, we have that $\Vert{\nabla\oL_t(h)}\Vert \leq \Vert{\tilde{z}_t}\Vert$.

Finally, it could be verified via straightforward calculations that the Hessian matrix of $\oL_t(h)$, i.e., $\nabla^2\oL_t(h)$, is positive semidefinite and hence, for all $t\in[T]$, $\oL_t(h)$ is a convex function. 

We now turn to introduce our main algorithmic tool, the Online Gradient Descent algorithm (OGD), which guarantees sublinear regret for adversarially-chosen loss functions in repeated games. Both the algorithm and its regret guarantees (in the context of our setting) are recorded in the following lemma.

\begin{lemma}{[Online Gradient Descent (see e.g., \cite{Hazan16})]}\label{lem:ogd}
Consider a sequence $\left((\tilde{z}_t,\omega_t,\tilde{\mu}_t)\right)_{t\in[T]}$ such that $\sup_{\tilde{z}_t}\Vert{\tilde{z}_t}\Vert \leq Z$ for some $Z>0$, and let $(\oL_t(\cdot))_{t\in[T]}$ be the corresponding sequence of loss functions defined according to Eq. \eqref{eq:loss}. Consider a sequence of aggregation vectors $(h_t)_{t\in[T+1]}$ generated by the following sequential algorithm, known as Online Gradient Descent:
\begin{eqnarray}\label{eq:ogd}
h_1 \gets \mathbf{0}_n; \qquad \forall t\in[T]: \quad h_{t+1} \gets \Pi_{\Vert{\cdot}\Vert\leq{}W}\left[{h_t - \sqrt{\frac{4W^2}{Z^2T}}\nabla\oL(h_t)}\right],
\end{eqnarray}
where $\Pi_{\Vert{\cdot}\Vert\leq W}[z] : = W\frac{z}{\max\{\Vert{z}\Vert, W\}}$ denotes the Euclidean projection onto the Euclidean ball of radius $W$ centered at the origin, for some $W>0$.
Then, it holds that
\begin{eqnarray*}
\sum_{t=1}^T\oL_t(h_t) - \min_{h:\Vert{h}\Vert\leq{}W}\sum_{t=1}^T\oL_t(h) = O\left({WZ\sqrt{T}}\right).
\end{eqnarray*}
\end{lemma}

Suppose that we choose $W = \sigma^{-1}\sqrt{n}$, and that this choice indeed satisfies $W \geq \Vert{A_{L}^{-1\top}\one_m}\Vert = \Vert{h^*}\Vert$ (which we verify in the sequel). We have by a straightforward application of Lemma \ref{lem:ogd}, followed by taking expectation over both sides of Eq. \eqref{eq:ogd} w.r.t. the randomness in the sequence $(\omega_t)_{t\in[T]}$, and applying Eq. \eqref{eq:expectedLoss} w.r.t. all $t\in[T]$, that
\begin{align*}
\sum_{t=1}^TL(h^*\cdot\tilde{z}_t + \tilde{\mu}_t)- \E\left[{\sum_{t=1}^TL(h_t\cdot\tilde{z}_t+\tilde{\mu}_t)}\right]  = O\left({WZ\sqrt{T}}\right),
\end{align*}
where $W,Z$ are as defined in Lemma \ref{lem:ogd}.

Thus, it only remains to bound the expressions $Z=\sup_{\tilde{z}_t}\Vert{\Lz_t}\Vert$, and $W=\sup_{A: \sigma_{min}(A)\geq \sigma}  \Vert{A_l^{-1\top}\one_m}\Vert^2$.

Since forecasts are non-extreme, we know that $F\in [\tau,1-\tau]^n$. This implies that $\Lf \in [\log(\tau)-\log(1-\tau),\log(1-\tau)-\log(\tau)]^n \subset [\log(\tau),-\log(\tau)]^n$, which implies that $\Lz_t\in [\log(\tau)+\log(\beta),-\log(\tau)-\log(\beta)]^n$. We recall that $\tau=T^{-1/2}$ and $\beta$ is a constant, which implies that $Z^2 =O(n \log^2 T)$ and $Z =O(\sqrt{n} \log T)$.

We turn to verify that for our choice of $W$ it indeed holds that $\Vert{h^*}\Vert_2 \leq W$. Let $A = U\Sigma{}V^{\top}$ denote the singular value decomposition of $A$. Also, in the following,  for any matrix $M$, we let $\Vert{M}\Vert_2$ denote the spectral norm of $M$, i.e., its largest singular value . Using standard manipulations we have that
\begin{eqnarray*}
\Vert{A_l^{-1\top}\one_m}\Vert_2^2 &=& \Vert{A(A^{\top}A)^{-1}\one_m}\Vert_2^2 \leq \Vert{A(A^{\top}A)^{-1}}\Vert_2^2\cdot\Vert{\one_m}\Vert_2^2 \\
&=&m\cdot\Vert{U\Sigma{}V^{\top}(V\Sigma^2V^{\top})^{-1}}\Vert_2^2 = m\cdot\Vert{\Sigma^{-1}}\Vert_2^2 = \frac{m}{\sigma_{\min}(A)^2} \leq \frac{n}{\sigma^2} = W^2,
\end{eqnarray*}
where the last inequality holds using the fact that $A$ is injective, which implies that $m \leq n$, and our assumption that $\sigma_{\min}(A) \geq \sigma$.

Summarizing, the total regret is bounded by $O(\sqrt{n} \log T \cdot \sqrt{n} \sigma^{-1}\cdot \sqrt{T})=\tilde{O}(n  \sigma^{-1} \sqrt{T})$.

\subsubsection{Algorithm for extreme realizations}\label{sec:ext}
In typical information/forecast aggregation problems the fact that experts might have very confident forecast \emph{helps} in aggregating the information (see e.g., \cite{SS,WiseC}). For instance, in the extreme case where one of the Bayesian experts assigns a probability of $1$ to the events that the state is $\omega=1$, the aggregator knows (w.p. 1) that the state is indeed $\omega=1$. The approximate analogue of this statement (i.e., in case expert's forecast is very close to 1 rather than 1) is the following.

Let $I\in \Delta(\Omega\times S_1\times ... \times S_n)$ be an \emph{arbitrary} information structure (not necessarily a partial evidence information structure), when $S_i$ is the signal set of expert $i$. We define the event $\zeta_0(\alpha)$ to be the event where at least one expert's forecast is in $[0,\alpha]$. Formally, 
\begin{align*}
\zeta_0(\alpha)=\{(\omega,s_1,...,s_n):\exists i\in [n] \text{ s.t. } \frac{I(\omega=1,s_i)}{I(\omega=1,s_i)+I(\omega=0,s_i)}\leq \alpha \}.
\end{align*}

\begin{lemma}\label{lem:extreme}
For every $\alpha>0$ and every information structure\footnote{The linear dependence on $n$ is unavoidable, even if we restrict attention to conditionally i.i.d. information structures. If $\alpha=\frac{1}{n}$ and the i.i.d. information structure has prior $\frac{1}{4}$ and binary posteriors $\frac{1}{n}$ and $\frac{1}{2}-\frac{1}{n}$ with equal ex-ante probability $\frac{1}{2}$, it is easy to verify that conditional on having an expert who predicts $\frac{1}{n}$ there is a constant probability of $\approx \frac{1}{12e+1}$ for $\omega=1$.} $I$, $I(\omega=1 | \zeta_0(\alpha))\leq n\alpha$. Namely, conditioned on the event that at least one expert has an extreme forecast in $[0,\alpha]$, the actual probability of $\omega=1$ is at most $n\alpha$.

\end{lemma}
The proof of the lemma is relegated to Appendix \ref{ap:ext-lem}. 

We define the event $\zeta_1(\alpha)$ to be the event where at least one expert's forecast is in $[1,1-\alpha]$, and we have the analogous Lemma.

\begin{lemma}\label{lem:extreme1}
$I(\omega=0 | \zeta_1(\alpha))\leq n\alpha$. 
\end{lemma}

From Lemmas \ref{lem:extreme} and \ref{lem:extreme1} can be deduce the following.
\begin{lemma}\label{lem:2ext}
$I(\zeta_0(\alpha) \cap \zeta_1(\alpha))\leq 2n\alpha$.
\end{lemma}
The proof is relegated to Appendix \ref{ap:ext-lem}.

We utilize Lemmas \ref{lem:extreme},\ref{lem:extreme1} and \ref{lem:2ext} to provide a good forecast in case of extreme realization of forecasts \emph{without any learning}. 

In the event $\zeta_0(\tau)\setminus \zeta_1(\tau)$ (i.e., an extreme forecast in $[0,\tau]$ and no extreme forecast in $[1-\tau]$) the algorithm forecasts $r=n\tau$. In the event $\zeta_1(\tau)\setminus \zeta_0(\tau)$ (i.e., an extreme forecast in $[1-\tau,1]$ and no extreme forecast in $[1-\tau,1]$) the algorithm forecasts $r=1-n\tau$. In the event $\zeta_0(\tau)\cap \zeta_1(\tau)$ (i.e., extreme forecasts from both sides) the algorithm forecasts $r=\frac{1}{2}$. The bounds of Lemmas \ref{lem:extreme},\ref{lem:extreme1} and \ref{lem:2ext} yield the following bound on the \emph{logarithmic loss} in each period:
$$
2[-1\cdot \ln(1-n\tau)-n\tau \ln(n\tau)] - 2n\tau \ln(\frac{1}{2}) = O(n\tau \ln(\frac{1}{n\tau}))$$
We recall that $\tau=T^{-1/2}$ which implies that the total logarithmic loss is bonded by $O(n T^{-1/2} \log T \cdot T)=\tilde{O}(n\sqrt{T})$.
This, in particular, implies the same bound on the expected \emph{regret}.


\subsection{Proof of Theorem \ref{th:stat}}
As in the proof of Theorem \ref{th:stat} we focus on the only interesting case where the static prior $\mu\in [\beta,1-\beta]$ is bounded away from the boundaries; i.e., $\beta>0$ is a constant\footnote{The case of extreme prior is not interesting because if the estimation of the prior is close to the boundaries one can apply similar (and in fact simpler) arguments to those of Section \ref{sec:ext} to deal with such a case.}. 

The algorithm splits into two phases. Phase 1 is used to estimate the prior, and the Phase 2 uses the estimator from Phase 1, and simply applies the algorithm from the proof of Theorem \ref{th:dyn}.

Phase 1 takes $T_1=n\sigma^{-1}\sqrt{T}$ periods. The estimator $\hat{\mu}$ is simply the fraction of periods where the state was $\omega_t=1$. Along Phase 1 the algorithm constantly forecasts $r\equiv\frac{1}{2}$, which yields a total logarithmic loss of $T_1\ln 2=O(T_1)$. We set $\delta=T_1^{-1/2}\ln(T)=\tilde{O}(n^{-1/2}\sigma^{1/2}T^{-1/2})$. By Hoeffding's inequality, after Phase 1, the probability that the estimator $\hat{\mu}$ deviates from the truth $\mu$ by more than $|\hat{\mu}-\mu|> \delta$ is at most $2\exp(-2T_1\delta^2)=O(T^{-2})$.

Phase 2 takes the remaining $T_2=T-T_1$ periods.
First let us bound the loss of the algorithm in case the estimator is wrong (i.e., $|\hat{\mu}-\mu|> \delta$). 
We bound from the boundaries $\{0,1\}$ the forecast $r_t$ made by the algorithm.
If the profile of forecasts $F$ is extreme, then $r_t\in [n\tau,1-n\tau]$.
If the profile of forecasts $F$ is not extreme, then in the log-likelihood space we can bound $||h||\leq W$ and $||z||=O(\sqrt{n}\log T)$. Therefore $|\tilde{r}_t|=O(\sqrt{n} \sigma^{-1}\cdot \sqrt{n}\log T)=O(\sqrt{n}\sigma^{-1} \log T)$. Therefore, there exists a constant $c>0$ such that $r_t\in[(1+\exp(c n\sigma^{-1}\log T ))^{-1}, 1-(1+\exp(c n\sigma^{-1}\log T ))^{-1}]$. We denote $\alpha=(1+\exp(c n\sigma^{-1}\log T ))^{-1}$.

Note that $[\alpha,1-\alpha] \supset [n\tau,1-n\tau]$. Therefore we deduce that in any case $r_t\in [\alpha,1-\alpha]$. In every single period the logarithmic loss is bounded by $-ln(\alpha)=O(n\sigma^{-1} \log T)$, which implies that the total loss is Phase 2 is bounded by \newline $O(n\sigma^{-1} T_2 \log T_2)$. The probability of a wrong estimator is $O(T^{-2})$. Therefore, the event of a wrong estimator adds to the expected loss at most \newline $O(T^{-2} n\sigma^{-1} T_2 \log T_2)=O(1)$.

Now we argue that in case the estimator is correct (i.e., $|\mu-\hat{\mu}|\leq \delta$), and we apply the algorithm in the proof of Theorem \ref{th:dyn} with an estimator $\muh$ then during Phase 2 the algorithm guarantees an expected regret of at most $\tilde{O}(n\sigma^{-1} T)$. 

The log-likelihood function $f(x)=\ln(\frac{x}{1-x})$ has a bounded derivative $|f'(x)|\leq \frac{1}{\beta(1-\beta)}$ in the segment $[\beta,1-\beta]$. Therefore, $|\mu-\hat{\mu}|\leq \delta$ guarantees that $|\tilde{\mu}-\hat{\tilde{\mu}}|\leq \frac{1}{\beta(1-\beta)} \delta = O(\delta)$. 

Our starting point is the log-likelihood aggregation problem (see Definition \ref{def:ll}) whose formulation is valid for the prior-ignorant aggregator. Unlike the prior-aware aggregtor, the prior ignorant aggregator \emph{cannot} deduce $\tilde{z}= A\tilde{y}$ from $A\tilde{y}+\tilde{\mu}\one_n $. However, he has an estimator $\hat{z}$ to $\tilde{z}$ which satisfies $||\hat{z} - \tilde{z}||_\infty \leq O(\delta)$. Now we bound $|L(h,\tilde{z})-L(h,\hat{z})|$ (see equation \eqref{eq:loss} for the definition of $L(h,z)$). 

Note that $|h\tilde{z}+ \tilde{\mu} - (h \hat{z} + \hat{\tilde{\mu}})| \leq (h \one_n +1) \frac{1}{\beta(1-\beta)}\delta =O(\sqrt{n}W\delta)$. Therefore $|\Bay(\tilde{z})-\Bay(\hat{z})|=O(\sqrt{n}W\delta)$ (because the derivative of the function $f(x)=(1+\exp{-x})^{-1}$ is bounded by $|f'(x)|\leq \frac{1}{2}$). Similarly we note that $|\ln(1+\exp(h\tilde{z}+ \tilde{\mu}) - \ln(1+\exp(h \hat{z} + \hat{\mu})|=O(\sqrt{n}W\delta)$ (because the derivative of the function $f(x)=\ln(1+\exp(x))$ is bounded by $|f'(x)|\leq 1$). By putting all these bounds together, we get that 
\begin{align}\label{eq:loss-lip}
|L(h,\tilde{z})-L(h,\hat{z})|=O(\sqrt{n}W\delta)
\end{align}

We define the hypothesis class $H=\{\hat{z}\rightarrow{}h\cdot\hat{z}: h\in \R^n,~\Vert{h}\Vert\leq W\}$.
Note that the optimal hypothesis $h^*=\one_m A_l^{-1}$ \emph{does not} depend on the prior, and thus belongs to this class. By Lemma \ref{lem:ogd} the Online Gradient Descent algorithm guarantees a regret of at most $O(WZ\sqrt{T_2})$ with respect to the aggregation $h^* \rightarrow h^* \hat{z} $. By Equation \ref{eq:loss-lip} this aggregation function defers from the optimal one by at most $O(\sqrt{n}W\delta)$. Therefore, the maximal total regret is $$O(WZ\sqrt{T_2}+\sqrt{n}WT_2\delta)=\tilde{O}(n\sigma^{-1}\sqrt{T}+n \sigma^{-1} T n^{-\frac{1}{2}} \sqrt{\sigma} 
T^{-\frac{1}{2}})=\tilde{O}(n\sigma^{-1}\sqrt{T}).$$

The overall expected regret can be bounded by $$O(n\sigma^{-1}\sqrt{T}) + O(T^{-2}) + \tilde{O}(n\sigma^{-1}\sqrt{T})=\tilde{O}(n\sigma^{-1}\sqrt{T})$$ where the first term is the regret during Phase 1, the second term is the regret in case of wrong estimation, and the last term is the regret in case of correct estimation.

\subsection{Proof of Proposition \ref{pro:imp}}\label{sec:imp1-pr}
As in Example \ref{ex:imp}, we consider an aggregtor who knows the prior $\mu=0.5$ (and thus $\LPri =0$), knows the evidence matrix $A$, knows the distribution of signals $C$, but observes only $F(s)$ rather than $s$. The above information is an upper bound on what the aggregator can learn.

Let $\tilde{z}$ be a vector such that $A\tilde{z}=\textbf{0}_n = A\textbf{0}_m$, and $\tilde{z}\cdot \one_m \neq 0$. Set the distribution of signals $C=(C_j)_{j\in [m]}$ to have the posterior beliefs $\frac{1}{2}$ (with log-likelihood 0) and $z_j=\logit(\tilde{z}_j)$ with strictly positive probability. This is possible by the Aumann Maschler's Splitting Lemma \cite{AM}, which states that any distribution over posteriors whose expectation is the prior can be implemented by some signals. In particular in our case, if w.l.o.g. $z_j<\frac{1}{2}$, by adding a third posterior belief $z'_j>\frac{1}{2}$ we can set the weights on $\{z_j,\frac{1}{2},z'_j\}$ such that the expectation will turn out to be $\frac{1}{2}$.

With positive probability, the profile of log-likelihoods will be $\Ly=\textbf{0}_m$, which yields a report of $A \textbf{0}_m=\textbf{0}_n$.
With positive probability, the profile of log-likelihoods will be $\Ly=\tilde{z}$, which yields the same report $A \tilde{z}=\textbf{0}_n$.
The optimal aggregation is $\textbf{0}_m\cdot \one_m=0$ in the first case, and it is $\tilde{z}\cdot \one_m \neq 0$ in the second case. Namely aggregator cannot perform as good as the optimal aggregation in the case he observes the report $\textbf{0}_n$, which occurs with positive probability. Thus aggregtor's expected regret is strictly positive in each period.

\subsection{Proof of Proposition \ref{pro:imp-dyn}}
In \cite{ABS} Proof of Theorem 4, the authors introduce a \emph{one-shot} mixed strategy (for the adversary) over conditionally independent information structures for which no prior-independent aggregation function can get a \emph{square loss} below $\frac{1}{4}-\epsilon$ whereas the optimal aggregation has a square loss of at most $\epsilon$.
It is easy to check that for this mixed strategy no prior-independent aggregation function can get a \emph{logarithmic loss} below\footnote{The number $\frac{1}{4}$ is the loss of the forecast $r=\frac{1}{2}$ in the square loss, whereas $\ln 2$ is the loss of the forecast $r=\frac{1}{2}$ in the logarithmic loss.} $\ln 2-\epsilon$ whereas the optimal aggregation has a square loss of at most $\epsilon$.

Consider an adversary who in each period chooses the information structure according to this mixed strategy. Then the adversary ensures that the expected regret (i.e., difference between the loss of the optimal aggregation and the loss of the algorithm) in each period is $\ln 2-\epsilon$ and thus the expected total regret after $T$ periods is $T$.     

%

\bibliographystyle{ACM-Reference-Format}
\bibliography{bib}

\appendix

\section{Proof of the Lemmas on Extreme Forecasts}\label{ap:ext-lem} 
\begin{proof}[Proof of Lemma \ref{lem:extreme}]
Without loss of generality we may assume that $|S_i|=2$ for every $i\in [n]$, because otherwise, we can merge all signals of player $i$ that yield a posterior in $[0,\alpha]$ into one signal $s^0_i$, and merge  all signals of player $i$ that yield a posterior in $(\alpha,1]$ into one signal $s^1_i$. It is easy to verify that if the  lemma holds for the new binary information structure for it yields the same result for $I$.

For a binary-signal information structure, we identify each signal profiles $s=(s_i)_{i\in [n]}$ with the subsets $B\subset [n]$ where $B=\{i\in [n]:s_i=s_i^1\}$. We denote $I(\omega=0,B)=a^0_B$ and $I(\omega=1,B)=a^1_B$, where $\sum_B a^0_B+a^1_B = 1$. The event $\zeta_0(\alpha)$ is equivalent to the event $B\neq [n]$, because this is the only signal profile where no expert has an extreme forecast. The conditional probability is given by 
$$I(\omega=1|\zeta_0(\alpha))=\frac{\sum_{B\subset [n], B\neq [n]} a^1_B}{\sum_{B\subset [n], B\neq [n]} a^0_B + a^1_B}.$$. 

Therefore $\max_I I(\omega=1|\zeta_0(\alpha))$ can be written as 
\begin{align}\label{eq:max}
\begin{split}
\max_{(a^\omega_B)_{B\subset [n],\omega=0,1}} & \frac{\sum_{B\subset [n], B\neq [n]} a^1_B}{\sum_{B\subset [n], B\neq [n]} a^0_B + a^1_B} \\
\text{subject to } & \ a^\omega_B \geq 0 \ \ \forall B\subset [n], \ \omega=0,1, \\
&\  \frac{\sum_{B\subset [n], i\notin B} a^1_B}{\sum_{B\subset [n], i\notin B} a^0_B + a^1_B}\leq \alpha \ \ \forall i\in [n],
\end{split}
\end{align}
where the second constrain corresponds to the fact that the posterior of expert $i$ after observing $s^0_i$ is at most $\alpha$.

First, we argue that there exists a point $a$ that maximizes \eqref{eq:max} where $a^0_B = 0$ for every $B\neq \emptyset$. Otherwise, if $a^0_B = c$, we can replace $\overline{a^0_B} = 0$ and $\overline{a^0_\emptyset} = a^0_\emptyset + a^0_B$. Note that the objective function is identical for $a$ and for $\overline{a}$, moreover it is easy to verify that the constrains remain valid.

Second, we argue that there exist a point $a$ that maximizes \eqref{eq:max} where $a^1_\emptyset = 0$. Otherwise, if $a^1_\emptyset = c$, we can replace $\overline{a^0_\emptyset} = 0$ and $\overline{a^1_{\{i\}}} = a^1_{\{i\}}+c$ for every $i\in [n]$. Note that the constrains remain unchanged whereas the objective function strictly increases.

Third, we note that the maximization problem \eqref{eq:max} is homogeneous of degree 0. Therefore we can w.l.o.g. set $a^0_\emptyset=1-\alpha$.

For the domain of $a$ that satisfies the first and the third property the constrains in \eqref{eq:max} can be simply written as $\sum_{B\subset [n], i\notin B} a^1_B \leq \alpha$. The second property guarantees that $$\sum_{B\subset [n], B\neq [n]} a^1_B \leq \sum_{i\in [n]} \ \sum_{B\subset [n], i\notin B} a^1_B \leq n\alpha,$$ because every element $a^1_B$ appears at least once in the second summation except of $a^1_\emptyset$, but $a^1_\emptyset=0$ by the second property.

Therefore, the objective functions is bounded by
$$\frac{\sum_{B\subset [n], B\neq [n]} a^1_B}{\sum_{B\subset [n], B\neq [n]} a^0_B + a^1_B} \leq \frac{n\alpha}{(1-\alpha)+ n\alpha}\leq n\alpha.$$
\end{proof}

\begin{proof}[Proof of Lemma \ref{lem:2ext}]
Since $I(\omega=0|\zeta_0 \cap \zeta_1)+I(\omega=1|\zeta_0 \cap \zeta_1)=1$, one of these expressions should be at least $\frac{1}{2}$. W.l.o.g. we assume that $I(\omega=1|\zeta_0 \cap \zeta_1)\geq \frac{1}{2}$. By Lemma \ref{lem:extreme} (in the other case we use Lemma \ref{lem:extreme1}) and Bayes rule we have $$n\alpha \geq  I(\omega=1|\zeta_0)\geq \frac{I(\zeta_0 \cap \zeta_1) }{I(\zeta_0)} I(\omega=1|\zeta_0 \cap \zeta_1)$$ which implies that $$I(\zeta_0 \cap \zeta_1)\leq n\alpha \frac{I(\zeta_0)}{I(\omega=1|\zeta_0 \cap \zeta_1)}\leq 2n\alpha.$$
\end{proof}

\end{document}